\documentclass{article}

\usepackage{microtype}
\usepackage{graphicx}
\usepackage{subcaption}
\usepackage{booktabs} 

\usepackage{hyperref}
\usepackage{bm}
\usepackage{dsfont}
\usepackage{cancel}

\usepackage[preprint]{icml2026}

\usepackage{amsmath}
\usepackage{amssymb}
\usepackage{mathtools}
\usepackage{amsthm}
\usepackage[normalem]{ulem}

\usepackage[capitalize,noabbrev]{cleveref}

\theoremstyle{plain}
\newtheorem{theorem}{Theorem}[section]
\newtheorem{proposition}[theorem]{Proposition}

\theoremstyle{definition}
\newtheorem{definition}[theorem]{Definition}

\theoremstyle{remark}

\newcommand{\changesto}[2]{#1}

\usepackage[textsize=tiny]{todonotes}

\icmltitlerunning{Submission and Formatting Instructions for ICML 2026}

\begin{document}

\twocolumn[
    \icmltitle{Exploiting Subgradient Sparsity in Max-Plus Neural Networks}



  \icmlsetsymbol{equal}{*}

  \begin{icmlauthorlist}
    \icmlauthor{Ikhlas Enaieh}{yyy}
    \icmlauthor{Olivier Fercoq}{yyy}
  \end{icmlauthorlist}

  \icmlaffiliation{yyy}{Image, Data, Signal Department (IDS), Telecom Paris, Institut Polytechnique de Paris, France}

  \icmlcorrespondingauthor{Ikhlas Enaieh}{ikhlas.enaieh@telecom-paris.fr}
  \icmlcorrespondingauthor{Olivier Fercoq}{olivier.fercoq@telecom-paris.fr}

  \icmlkeywords{Machine Learning, ICML}

  \vskip 0.3in
]



\printAffiliationsAndNotice{}  

\begin{abstract}
Deep Neural Networks are powerful tools for solving machine learning problems, but their training often involves dense and costly parameter updates. In this work, we use a novel Max-Plus neural architecture in which classical addition and multiplication are replaced with maximum and summation operations respectively. This is a promising architecture in terms of interpretability, but its training is challenging. 
A particular feature is that this algebraic structure naturally induces sparsity in the subgradients, as only neurons that contribute to the maximum affect the loss. However, standard backpropagation fails to exploit this sparsity, leading to unnecessary computations. In this work, we focus on the minimization of the worst sample loss which transfers this sparsity to the optimization loss. 
To address this, we propose a sparse subgradient algorithm that explicitly exploits the algebraic sparsity. By tailoring the optimization procedure to the non-smooth nature of Max-Plus models, our method achieves more efficient updates while retaining theoretical guarantees. This highlights a principled path toward bridging algebraic structure and scalable learning.
\end{abstract}

\section{Introduction}
Deep Neural Networks (DNNs) have achieved remarkable success in tasks ranging from computer vision to natural language processing, due to their ability to learn complex model patterns from high-dimensional data \citep{deepLearning}. However, this expressiveness comes with significant computational cost: training such models typically involves dense updates to millions of parameters~\citep{dense_NN_computationalCost}, regardless of how many actually influence the model's output for a given sample. This inefficiency motivates the search for architectures and training algorithms that can exploit \changesto{additional structural properties of neural networks to reduce redundant computation, without compromising performance.}{sparsity without compromising performance.} \\
In this work, we focus on neural networks built using Max-Plus and Min-Plus algebras~\cite{morphological_NN}. Unlike traditional neurons that compute weighted sums of inputs,
\changesto{ these architectures rely on selection-based operations. A Max-Plus neuron replaces addition with a maximum and multiplication with summation, producing an output of the form
$\hat{y} = \max_j \{ x_j + w_j \},$
while a Min-Plus neuron computes the corresponding minimum,
providing a complementary behavior. In both cases, only the inputs attaining the maximum or minimum contribute to the neuron’s output, while all others are inactive.}{ Max-Plus neurons replace addition with a maximum operation and multiplication with summation. The output of a Max-Plus neuron is
$\hat{y} = \max_j \{ x_j + w_j \}$
, so only the inputs attaining the maximum affect the output, creating naturally sparse activations. Min-Plus neurons operate as
$\hat{y} = \min_j \{ w_j + x_j \}$ providing complementary behavior.}  These algebraic operations induce subgradients that are inherently sparse, offering a promising foundation for more efficient learning algorithms.
\\ \\
Unfortunately, conventional backpropagation and automatic differentiation frameworks are not optimized for such sparse, non-smooth structures. They compute all the coordinates of the gradients irrespective of the model's sparsity ~\citep{denseGradient}, resulting in redundant computation and limited scalability~\citep{computationalcomplexitylearningneural}. In contrast, the derivative information in Max-Plus and Max-Minus networks takes the form of a sparse subgradient: only the weights lying on the active paths, those corresponding to inputs that attain the maximum or minimum, receive nonzero updates, while all other coordinates remain zero~\citep{sparsity_maxplus_structure}. \changesto{However, standard optimization methods fail to exploit this property: they treat these models as dense and propagate updates to all parameters,}{Standard optimization methods fail to exploit this property, updating every parameter} instead of focusing computation where it matters. To overcome these limitations, we develop a sparse subgradient training algorithm tailored to the non-convex, non-smooth nature of Max-Plus/Min-Plus neural networks, enabling updates only along the active computational paths.

Our interest in these architectures stems from their ability to naturally induce sparsity in both forward and backward passes~\citep{sparsity_maxplus_structure}, making them appealing alternatives to dense networks. The $(\mathrm{max},+)$ and $(\mathrm{min},+)$ algebras have been applied in previous works, achieving competitive performance in various tasks. For instance, max-plus operators have been used for filter selection and model pruning~\citep{filter_selection}, and the Min-Max-Plus architecture has been shown to be a universal approximator for continuous functions~\citep{minmaxplus}. These results indicate that moving toward structured sparse architectures does not necessarily sacrifice expressivity.

\section{Problem Setup and Notation}

We begin by recalling the basic building block of our model.
\begin{definition}[Morphological Perceptron {\normalfont \citep{morphological_activation}}]
Given an input vector $\mathbf{x} \in \mathbb{R}^{N}_{\max}$ (with $\mathbb{R}_{\max} = \mathbb{R} \cup \{-\infty\}$), 
a weight vector $\mathbf{w} \in \mathbb{R}^{N}_{\max}$, and a bias $b \in \mathbb{R}_{\max}$, 
the morphological perceptron computes its activation as
\[
a(\mathbf{x}) = \max \left\{ b, \; \underset{1 \leq i \leq {\color{black}N}\ 
}{\max}  \{ x_i + w_i \} \right\},
\]
where $x_i$ (resp.\ $w_i$) is the $i$-th component of $\mathbf{x}$ (resp.\ $\mathbf{w}$).
\end{definition}

This perceptron highlights the distinctive feature of $(\max,+)$ algebra: 
only the inputs that achieve the maximum contribute to the activation, 
naturally leading to sparse forward computations. 
To study how this sparsity propagates during training, 
we now turn to the loss function and its subgradients.

\subsection{Loss Function and Sparsity Motivation}

Following~\citep{filter_selection}, we study the Linear Max-Plus model: 
a linear layer with ReLU activation, followed by a $(\max,+)$ layer 
and a softmax output. 
Let $z_{i,d}$ denote the pre-softmax score for sample $i$ and class $d$, and
\[
\hat{y}_{i,d} = \frac{\exp(z_{i,d})}{\sum_{d'=1}^C \exp(z_{i,d'})}
\]
be the predicted probability of class $d$, with $y_i \in \{1,\dots,C\}$ the true label.

We compare the sparsity of the subgradients of the 
Categorical Cross-Entropy (CCE) loss in two settings:

\begin{enumerate}
\item \textbf{Single-sample loss.} The loss for a randomly selected sample 
$i \sim \mathcal{U}(1,N)$ is
\[
\text{Loss}_i(w) = - \sum_{d=1}^C 
\mathds{1}_{\{y_i = d\}} \, \log (\hat{y}_{i,d}) .
\]

\item \textbf{Average loss.} The mean CCE across all $N$ samples is
\[
\text{Loss}_A (w)= \frac{1}{N} \sum_{i=1}^N \text{Loss}_i(w) .
\]
\end{enumerate}
\changesto{As a consequence of the max structure of the perceptron, the subgradients of these losses are sparse, resulting in only partial updates of the model parameters.}{
The max structure of the perceptron ensures that the subgradients of these losses 
are inherently sparse: only the weights associated with the active inputs (those 
attaining the maximum) receive nonzero updates.} 
To measure this effect, we adopt the sparsity metric~\citep{nesterov_subgradient}
\[
\gamma(x) = \frac{\text{number of non-zero elements in } x}{\dim(x)},
\]
which provides insight into the number of non-zero entries in the vector $x$.
As an example, we considered a morphological perceptron model initialized with i.i.d. Glorot uniform parameters~\cite{glorot2010understanding} and computed the sparsity level of a subgradient on the MNIST dataset~\citep{lecun2002mnist}.
We compared the subgradient sparsity for the average loss and the average subgradient sparsity for one image:
    $$\gamma \left( \cfrac{1}{N} \sum_{i=1}^{N} \cfrac{\partial \text{Loss}_{i}}{\partial w} \right) = 0.82, \;\;\cfrac{1}{N} \sum_{i=1}^{N} \gamma \left(\cfrac{\partial \text{Loss}_{i}}{\partial w} \right) = 0.048  $$
\changesto{These findings demonstrate that effectively exploiting the sparsity induced by Max-Plus structures requires training strategies based on individual examples rather than averaged updates.}{
These findings suggest that leveraging the sparsity induced by Max-Plus structures requires training algorithms that emphasize individual examples rather than uniform averages.} We therefore adopt a strategy that selects, at each iteration, the training sample with the \changesto{largest}{highest} loss. This \changesto{strategy leverages}{promotes} both subgradient sparsity, as shown in \citep{nesterov_subgradient}, and robustness, by concentrating learning on the network’s weakest predictions. \changesto{This leads us to consider the  following optimization problem}{The optimization problem we are thus going to try and solve is given by :} 
$$\min_w \mathcal{L}(w) = \min_w \underset{1 \leq i \leq {\color{black}N}\ 
}{\max} \text{Loss}_i(w) $$

While the average loss captures typical performance, the maximum loss directly controls the worst-classified sample. The following proposition shows that this control is in fact sufficient to ensure perfect classification on the training set.

\begin{proposition}[Perfect classification under a max-SCCE threshold]
\label{prop:maxloss}
If the maximum Sparse Categorical Cross-Entropy loss is strictly less than $\log 2$, then the model achieves $100\%$ classification accuracy on the training set.
\end{proposition}

The proof is provided in Appendix~\ref{proposition_proof}.

Nevertheless, computing the maximum loss over a large dataset at each step would require $\mathcal{O}(N)$ operations, which is computationally expensive. To mitigate this, we introduce a Short Computational Tree (SCT) structure that enables logarithmic-time updates and maximum tracking \citep{nesterov_subgradient}. This makes our max-loss formulation scalable to realistic datasets.

\subsection{Short Computational Tree}

\begin{definition}[Short Computational Tree (SCT) {\normalfont \citep{nesterov_subgradient}}]
Given an input vector $x \in \mathbb{R}^{\color{black}N}$ with ${\color{black}N} = 2^k$, $k \geq 1$, 
the Short Computational Tree (SCT) is a binary tree of height $\log_2 {\color{black}N}$ and $k+1$ levels, 
constructed to compute symmetric functions (e.g., $\max$) hierarchically.  

At the base level, each leaf node contains one entry of the input vector $x$.  
Each internal node computes the binary maximum of its two children:
\begin{itemize}
    \item At level $0$, the tree holds the input values:
    \[
    v_{0,i} = x(i), \quad i = 1, \dots, {\color{black}N}.
    \]
    \item The internal nodes of the tree are computed recursively:
\[v_{i+1,j} = \psi_{i+1,j}(v_{i,2j-1}, v_{i,2j}),\] 
 $i = 0, \dots, k-1,\quad  j = 1, \dots, 2^{k-i-1}.$

    In our case, $\psi_{i,j}(a,b) = \max(a,b)$.
\end{itemize}
The root node contains the overall maximum.
\end{definition}

\changesto{This hierarchical construction computes the maximum through a structured sequence of pairwise comparisons, which can be reused across iterations. While the initial construction of the SCT has the same cost as a standard linear scan, its advantage becomes apparent once the maximum must be updated repeatedly. In our setting, each iteration modifies only a single input entry: let $x$ denote the current input vector and $\tilde{x}$ the updated vector after one iteration. The effect of this  change propagates along a unique path from the corresponding leaf to the root of the tree. 
Consequently, updating the maximum requires visiting only one node per level of the tree. 
}{Since we are dealing with the maximum of the loss, the subgradient will be sparse and we are going to update only a few values at each iteration. Assume $x \in \mathbb{R}^N$ and $\tilde{x}$ differs from $x$ in only one coordinate. It was shown in~\citep{nesterov_subgradient} how we can do these sparse updates efficiently by starting at a leaf of the short computational tree and processing only along the path to the root. The complexity comparison between this procedure and computing the maximum from scratch is shown in Table~\ref{tab:sct_complexity}.}
\begin{table}[h]
\centering
\vspace{0.5\baselineskip}  
\caption{Complexity of maximum computation: Standard Maximum vs. SCT representation} 
\vspace{0.5\baselineskip}  
\label{tab:sct_complexity}
\begin{center}
\renewcommand{\arraystretch}{1.2}
\begin{tabular}{llll}
\textbf{Operation} & \textbf{Target} & \textbf{Standard} & \textbf{SCT} \\
\hline \\
Compute & $\underset{1 \leq i \leq N}{\max} \; x_i$ 
        & $\mathcal{O}(N)$ 
        & $\mathcal{O}(N)$ \\
Update  & $\underset{1 \leq i \leq N}{\max} \; \tilde{x}_i$ 
        & $\mathcal{O}(N)$ 
        & $\bm{\mathcal{O}(\log N)}
$ \\
\end{tabular}
\end{center}
\end{table}

\changesto{As summarized in Table~\ref{tab:sct_complexity}, the use of the SCT reduces the cost of updates from $\mathcal{O}(N)$ to $\mathcal{O}(\log N)$, making it a significantly more efficient alternative for iterative max-based optimization.}{From Table~\ref{tab:sct_complexity}, we observe that using the SCT does not change the cost of computing the maximum, but it drastically reduces the cost of updates from $\mathcal{O}(N)$ to $\mathcal{O}(\log N)$ making it a far more efficient alternative.}

This efficiency is particularly important once we embed the max-loss formulation into concrete network architectures, starting with the zero hidden layer model.

\section{Model Architecture}
\subsection{Zero Hidden Layer Model}

We begin with a baseline architecture for classification, where the network maps the input directly to the output class scores using a single Max-Plus layer followed by a softmax activation. 

Formally, the training objective is
\[
\min_{W} \max_{1 \leq n \leq N} \left( -\log({\color{black}{\hat{y}}}_{n, y_n}) \right),
\]
where \changesto{$y_n \in \{1,\dots,C\}$}{$y_n$} is the true label for sample $n$ and the predicted probability ${\color{black}{\hat{y}}}_{n,d}$ for class $d$ is given by
\begin{equation}
{\color{black}{\hat{y}}}_{n,d} = \frac{\exp\!\left( \max_p \big({\color{black}{X}}_{n,p} + {\color{black}{W}}_{d,p}\big) \right)}
{\sum_{d'=1}^{C} \exp\!\left( \max_p \big({\color{black}{X}}_{n,p} + {\color{black}{W}}_{d',p}\big) \right)}.
\label{formula_Yhat_0-hidden}
\end{equation}
\changesto{Here, $X \in \mathbb{R}^{N \times P}$ is the data matrix with entries $X_{n,p}$ denoting the $p$-th feature of sample $n$, and $W \in \mathbb{R}^{C \times P}$ is the weight matrix.
}{Here, ${\color{black}{x}}_{n,p}$ denotes the $p$-th feature of sample $n$, and $W \in \mathbb{R}^{C \times P}$ is the weight matrix.}

Developing the expression yields
\begin{align}
\min_W \max_{n} \Big(
  &-\max_{p}\big({\color{black}{X}}_{n,p}+{\color{black}{W}}_{y_n,p}\big) \nonumber\\
  &+ \log \sum_{d=1}^{C} \exp\!\big(\max_{p}({\color{black}{X}}_{n,p}+{\color{black}{W}}_{d,p})\big)
\Big).
\end{align}

Because this objective involves multiple $\max(\cdot)$ terms, one for each class, the benefit of the \changesto{SCT}{Short Computational Tree (SCT)}, introduced earlier, becomes especially significant. After the initial $\mathcal{O}(N)$ computation at the first iteration, subsequent iterations require only $\mathcal{O}(\log N)$ updates, so that across training the vast majority of computations enjoy reduced complexity.

To fully leverage this computational advantage, \changesto{we now introduce}{we next turn to} the analytical framework 
required for differentiating through our nonsmooth architecture.

\subsection{Subgradient Computation}
\changesto{To handle the nonsmooth operators induced by the max-based structure of our model, we rely on the framework of conservative set-valued fields introduced by \citep{conservativesetvaluedfields}. This framework provides a generalized notion of differentiation that extends classical calculus to nonsmooth functions while preserving a consistent chain rule. Functions admitting such fields are referred to as path-differentiable, a class broad enough to include convex, concave, Clarke-regular, and semialgebraic Lipschitz continuous functions. In particular, it supports the composition of operators such as the maximum and the $\log\!\sum\exp$,  enabling backpropagation computations to be rigorously extended to our nonsmooth architecture. This framework forms the analytical foundation for the subgradient derivations that follow.}{To handle the nonsmooth operators that appear in our model, 
we rely on the framework of conservative set-valued fields~\citep{conservativesetvaluedfields}, 
a generalized derivative that extends classical calculus to nonsmooth functions. 
Functions equipped with such fields are called path differentiable, a class broad 
enough to cover convex, concave, Clarke-regular, and semialgebraic Lipschitz continuous 
functions. This guarantees that even in the presence of operators such as 
$\max$ and $\log\!\sum\exp$, a consistent chain rule applies. 
Hence, backpropagation-style computations can be rigorously extended 
to our nonsmooth setting, providing the foundation for the derivations below.}

To simplify notation, we define the critical indices. 
Let
\begin{align*} n^* \in \arg\max_n \Big(  &-\max_{p}\big({\color{black}{X}}_{n,p}+{\color{black}{W}}_{y_n,p}\big) \nonumber\\
  &+ \log \sum_{d=1}^{C} \exp\!\big(\max_{p}({\color{black}{X}}_{n,p}+{\color{black}{W}}_{d,p})\big) \Big) \end{align*}
be the index of one of the worst-classified sample, 
\[
p^*(d) \in \arg\max_p \big({\color{black}{X}}_{n^*,p} + {\color{black}{W}}_{d,p}\big)
\]
a maximizer over features for a given class $d$ and the worst-classified sample $n^*$.  
and $d^* := y_{n^*}$.

With this notation, the subgradient of the loss with respect to ${\color{black}{W}}_{i,j}$ becomes
\begin{align}
\frac{\partial \mathcal L}{\partial {\color{black}{W}}_{i,j}}(W)
= -\mathbb{I}(i = d^*) \mathbb{I}(j = p^*_{d^*})
   + \hat {\color{black}{y}}_{n^*, i}\;\mathbb{I}(j=p^*_i). \label{calculate_gradient_0hidden}
\end{align}
A \changesto{detailed}{complete} derivation is provided in appendix \ref{supp:subgrad-proof}.
Thus, the sparse subgradient matrix contains at most $C$ nonzero entries out of the total $C\times P$, with one entry per class $d \in \{1,\dots,C\}$ located at $(d, p^*_d)$, 
including the entry corresponding to the true class $d^*$. 

Although this formulation is elegant and highlights the role of sparsity, 
our experiments revealed that the loss plateaued at a relatively high value. 
This limitation motivates the introduction of a more expressive architecture 
with hidden layers, \changesto{presented in the next section.}{discussed next.}

\subsection{Model with One Hidden Layer}
We now consider a more expressive architecture inspired by Linear Min–Max (LMM) networks~\citep{minmaxplus}, applied sequentially to each input sample. We will first describe the model in the regression setting, as established in the original approximation theorem, then we will show how to extend it to a classification task.

\begin{itemize}
\item \textbf{Linear Layer:} Let ${\color{black}P}$ denote the number of input features and $x_n = X_{n,:}=(X_{n, 1}, \dots, X_{n,P})\in\mathbb{R}^{\color{black}P}$ the $n$-th sample.
We apply a sparse linear transformation
\[
\lambda(x) = [\,K_{0}x_1,\,-K_{1}x_1,\,\dots,\,K_{2{\color{black}P}-1}x_{\color{black}P},\,-K_{2{\color{black}P}}x_{\color{black}P}\,],
\]
where $K$ is a vector of parameters. This can be also written as
\[
\lambda(x) = W^0 x \in \mathbb{R}^{2{\color{black}P}}, 
\quad W^0 \in \mathbb{R}^{2P \times {\color{black}P}},
\]
where $W^0$ is a sparse matrix with predefined sparsity pattern.

    \item \textbf{Min-Plus Layer:} Let $ h = \{1,\dots,H_1\}$ denote the hidden neurons in this layer. Then :
    \[
    g_{h}(x) \;=\; \min_{i \in \{1,\dots,2{\color{black}P}\}} 
    \big( \lambda_i(x) + W^1_{i,h} \big),
    \]
    where $W^1 \in \mathbb{R}^{2{\color{black}P} \times H_1}$ is the (min,+) weight matrix.

\item \textbf{Max-Plus perceptron:}
The output of the LMM is given by
\[
\hat f_{W^0, W^1, W^2}(x) = \max_{h \in \{1,\dots,H_1\}} 
    \big( g_{h}(x) + W^2_{h} \big).
\]

A key theoretical foundation for our work is the following result from~\citep{minmaxplus}: 

\begin{theorem}[Universal Approximation of Lipschitz Functions using LMM Networks]
\label{thm:universal_approx}
Let $f $ be any Lipschitz-continuous function defined on a compact domain ${\color{black}\mathcal{X}} \subset \mathbb{R}^{\color{black}P}$. Then, there exists a sparse linear map $W^0$ and sequences of weight matrices $(W^{1,m})_{m}$ and $(W^{2, m})_{m}$ such that the {\color{black}corresponding} sequence of LMM networks $(\hat{f}_{W^0, W^{1, m}, W^{2, m}})_{\color{black}m \ge 1}$,  converges uniformly to $f$ as $m \to \infty$. That is,
\[
\lim_{m \to \infty} \sup_{x \in X} | \hat{f}_{W^0,W^{1, m},W^{2, m}}(x) - f(x) | = 0.
\]
\end{theorem}

 \changesto{This theorem provides the starting point for our approach. In particular, the
detailed proof in appendix \ref{ProofThmUniversalApprox} reveals that the transformation $\lambda$ can be chosen to be sparse, a property that directly motivates our parameter
initialization strategy. Specifically, $W^0$ depends only on the target function $f$, whereas the matrices $W^{1,m}$ and $W^{2,m}$ grow in size with the approximation parameter $m$.}{This theorem provides the starting point for our approach. We interpret the transformation \(\lambda\) not only as linear but also as \emph{sparse}, a perspective that we make explicit in the detailed proof included in appendix \ref{ProofThmUniversalApprox} Furthermore, the construction in the proof directly motivated our initialization strategy. In the proof of convergence, $W^0$ depends on $f$ only but $W^{1,m}$ and $W^{2,m}$ are growing in size with $m$.} For each $m$, we select $m$ points $(x^1, \dots, x^m)$ in $\mathcal{X}$ and each hidden neuron is responsible for interpolating $f$ around $x^m$. Our initialization, described formally in Section~\ref{sec:init}, corresponds to selecting a small number of samples and initializing the neural network parameters as in the interpolation result for this small number of samples.
This yields a principled and effective starting point for our model design.  

\changesto{While Theorem~\ref{thm:universal_approx} establishes an approximation result for
real-valued functions, we build on this framework to construct an LMM model for
vector-valued outputs $f : \mathbb{R}^{\color{black}P} \to \mathbb{R}^C$.}{Let us explain how we can extend the LMM, designed for functions from $\mathbb R^{\color{black}P}$ to $\mathbb R$, to classification tasks where $f:\mathbb R^{\color{black}P}\to\mathbb R^C$.}
We replace the single Max-Plus perceptron by a Max-Plus layer and a softmax activation.
    \item \textbf{Max-Plus Layer \& Softmax:} For each class $d \in \{1, \dots, C\}$, the class score is
    \[
    {\color{black}z}_{d}(x) \;=\; \max_{h \in \{1,\dots,H_1\}} 
    \big( g_{h}(x) + {\color{black}W}^2_{h,d} \big),
    \]
    with $W^2 \in \mathbb{R}^{H_1 \times C}$, and the predicted probabilities are
    \[
    \hat{{\color{black}y}}_{d}(x) \;=\; 
    \frac{\exp\!\big( {\color{black}z}_{d}(x) \big)}
         {\sum_{d'=1}^{C} \exp\!\big( {\color{black}z}_{d'}(x) \big)}.
    \]
\changesto{The corresponding
approximation guarantee and explicit constructions of $(W^{1,m})_m$ and
$(W^{2,m})_m$ are provided in appendix \ref{supp:pf-vs-exp}.}{\text We show in Section \ref{supp:pf-vs-exp} in the supplementary material that the approximation theorem extends to the classification case with an explicit formula for the sequences $(W^{1, m})_m$ and $W^{2, m}$.}
\end{itemize}

\subsection{Initialization}
\label{sec:init}
Let $(x_n, y_n) \in \mathbb{R}^{\color{black}P} \times \{1,\dots,C\}$ denote a labeled sample. 
We initialize the LMM network as a composition of three layers: 
a sparse linear transformation $\lambda$ with weight matrix $W^0 \in \mathbb{R}^{2{\color{black}P} \times {\color{black}P}}$, 
followed by a Min-Plus layer with weights $W^1 \in \mathbb{R}^{2{\color{black}P} \times H}$, 
and a final Max-Plus layer parameterized by $W^2 \in \mathbb{R}^{H \times C}$, where $H$ is the number of hidden neurons.
The transformation $\lambda$ encodes both positive and negative directions for each feature, defined by:
\[
W^0_{2i-1, i} = k, 
\quad 
W^0_{2i, i} = -k, 
\quad \text{for } i = 1, \dots, {\color{black}P},
\]
with all other entries of $W^0$ equal to zero. The scalar \( k \) serves as an initial scaling parameter and is adjusted during training through subgradient updates, effectively acting as a learnable quantity. The corresponding transformed input is then $\lambda(x_n) = W^0 x_n \in \mathbb{R}^{2{\color{black}P}}$.

The hidden layer weights $W^1$ are initialized using a subset of $H$ randomly chosen training samples $\{x_{n_h}\}_{h=1}^H$, ensuring interpolation:
\[
W^1_{2i-1, h} = -k \cdot x_{n_h, i}, 
\quad
W^1_{2i, h} = k \cdot x_{n_h, i}.
\]
The output layer weights $W^2$ encode class separation by boosting the true class and penalizing the others:
\[
W^2_{h,d} = 
\begin{cases}
k, & \text{if } d = y_{n_h}, \\
-k, & \text{otherwise}.
\end{cases}
\]
\changesto{This initialization associates each hidden neuron with a specific training
sample, while the output layer penalizes incorrect classes for that sample.
}{This initialization ensures that each of the $H$ neurons supports one specific training sample exactly, while penalizing incorrect classes.}

The scaling parameter \( k \) plays a central role in shaping the geometry and confidence of the network outputs. In our initialization, $k$ controls the slope of the piecewise-linear functions used in the approximation (as shown in the proof of Theorem~\ref{thm:universal_approx}). 

\subsection{Subgradient Derivation}
Define the training objective:
\[
\mathcal L(W) \;=\; \max_{1\le n\le N} \text{Loss}_n(W)
\;=\; \max_{1\le n\le N} \bigl(-\log \hat{y}_{y_n}(x_n)\bigr).
\]

Let $n^\star \in \arg\max_n \text{Loss}_n(W)$.

For each class $d$, choose
\begin{align*}
h_d^\star &\in \arg\max_{h} \Big( g_{h}(x_{n^*}) + W^2_{h,d} \Big), \\[6pt]
i_d^\star &\in \arg\min_{i} \Big( \lambda_{i}(x_{n^*}) + W^1_{i,h_d^\star} \Big)
\end{align*}
A subgradient of $\mathcal L$ at $W=(W^0, W^1, W^2)$ can be found as follows:

\noindent\textbf{(a) Subgradient w.r.t.\ $W^2$:}
The calculation is very similar to \eqref{calculate_gradient_0hidden}
\[
\frac{\partial \mathcal L}{\partial W^2_{h,d}} \;=\;
\mathbb I_{\{h = h^*_d\}}(\hat{y}_d(x_{n^*}) - \mathbb{I}_{\{d = y_{n^\star}\}})\;.\]

\noindent\textbf{(b) Subgradient w.r.t.\ $W^1$:}
\[
\cfrac{\partial \mathcal L}{\partial W^1_{i,h}} \;=\;
\begin{cases}
\cfrac{\partial \mathcal L}{\partial W^2_{h,d}}, & \text{ if } h = h_d^\star \text{ and } i = i_d^\star,\\[3pt]
0, & \text{otherwise.}
\end{cases}
\]

\noindent\textbf{(c) Subgradient w.r.t.\ $W^0$:}
\[
\cfrac{\partial \mathcal L}{\partial W^0_{i,p}} \;=\;
\begin{cases}
 \cfrac{\partial \mathcal L}{\partial W^2_{h^\star,d}} \, x_{n^\star,p_d}, & \text{if } i = i_d^\star\\[8pt]
0 & \text{otherwise.}
\end{cases}
\]
\changesto{A detailed proof of the subgradient derivation is given in appendix \ref{supp:subgrad-proof-LMM}.}{}\\
Interestingly, even after enriching the architecture with different types of neurons, the subgradient remains extremely sparse: it contains at most $C$ nonzero elements per layer, one for each class.

\begin{theorem}[Sparsity of the Subgradient]
For any LMM network trained with the maximum SCCE loss,  
the subgradient with respect to the parameter matrices contains at most $C$ nonzero elements per layer,  
where $C$ is the number of classes.
\end{theorem}

\section{Optimization Algorithms}
In order to train the Linear-Min-Max model, we are going to consider 2 algorithms: stochastic gradient descent on the average loss, sparse subgradient descent on the maximum loss.

\subsection{Stochastic gradient descent on the average loss}

This is the most natural choice and comes as a baseline for our other proposed algorithm. At each iteration, we consider a single sample, do the forward pass on this sample, which is the most computationally intensive part of each iteration, and then compute the subgradient of the loss for this sample. Note that this backward pass is very cheap thanks to the sparsity of the subgradient.

\subsection{Sparse subgradient descent on the maximum loss}

Building on Nesterov’s seminal work on subgradient methods for nonsmooth optimization, 
we extend the framework to our setting by integrating three key ingredients:  
(i) the Short Computational Tree (SCT) structure, which enables maximum-type 
operations to be updated in logarithmic time,  
(ii) Polyak’s adaptive step-size rule, which stabilizes convergence in 
nonsmooth optimization, and  
(iii) the natural sparsity of the parameters when initialized sparsely, 
which we exploit to reduce the cost of the updates. 

Formally, we consider the problem
\[
\min_{W = W^0, W^1, W^2} \mathcal L(W)  =\min_{W = W^0, W^1, W^2}\max_{1 \leq n \leq N} \text{Loss}_n(W)
\]
The sparse subgradient algorithm is then given by
\begin{align*}
&\text{Initialize } W_0 = (W_0^0, W_0^1, W_0^2)  \\
&W_{k+1} = W_k - \alpha_k \mathcal L'(W_k), 
\quad k \geq 0,
\end{align*}
where $\mathcal L'(W_k)$ denotes a sparse subgradient belonging to a conservative field, and 
\[
\alpha_k = \frac{\mathcal L(W_k)-\mathcal L^\ast}{\| \mathcal L'(W_k) \|^2}.
\]
Here, $\alpha_k$ is the Polyak step size, chosen adaptively at each iteration 
based on the current suboptimality $\mathcal L(x_k)-\mathcal L^\ast$ and the squared norm of the subgradient~\citep{polyakStepSize}, 
with $\mathcal L^\ast = \inf_W \mathcal L(W)$ denoting the optimal function value. Generally speaking, we don't know the value of $\mathcal L^*$, but 
in our setting, we target $\mathcal L^\ast = 0$, corresponding to the minimum of the 
Sparse Categorical Cross-Entropy loss.
We also consider a constant step size $\alpha_k = \alpha$ for a small value $\alpha$ in the final iterations to account for the fact that our estimation of $\mathcal L^\ast$ may be too optimistic.

\changesto{When initializing $W_0$, we perform a full forward pass and we store all the short computational trees. This requires memory but it then allows to perform the sparse updates as in \cite{nesterov_subgradient}. We can leverage sparsity in the dataset for the updates of the first layer (i.e. the $W^0$ matrix) and the sparsity of the forward pass to only update the SCTs that incur changes for the other layers.}{}

\section{Experiments}
\subsection{MLP vs. LMM on Iris}
We evaluate the proposed Linear-Min–Max (LMM) model on the Iris dataset and compare it against a standard Multi-Layer Perceptron (MLP). The data are randomly split into $70\%$ training samples and $30\%$ testing samples.
For a fair comparison, both models use the same depth and comparable width. The architecture composed of three layers with $(P,2P,H,C)$ neurons, where $P=4$ is the input dimension, $H=20$ the hidden-layer width, and $C=3$ the number of classes. All models are trained for $50{,}000$ iterations.

The training objectives differ between the two approaches.  
The MLP is trained by minimizing the average cross-entropy loss using the Adam optimizer with learning rate of $0.01$ (MLP ($\frac{1}{N}$) in Table~\ref{tab:iris_train}).  
The LMM model is trained either by minimizing the average loss using stochastic gradient descent (LMM ($\tfrac{1}{N}$)) or by minimizing the maximum loss using the proposed sparse subgradient algorithm (LMM ($\max_n$)).

Since the objectives are not identical, we report both the average loss and the maximum loss, computed over the final loss vector.

\begin{table}[h]
  \caption{\textbf{Training} performance on the Iris dataset.}
  \label{tab:iris_train}
  \begin{center}
    \begin{small}
      \begin{sc}
        \begin{tabular}{lcccr}
          \toprule
         Model& MLP ($\frac{1}{N}$)&LMM ($\frac{1}{N}$) & LMM ($\max_{n}$)  \\
          \midrule
          max loss & 1.839& 1.025 & 0.426  \\
          avg loss & 0.025 & 0.16 &  0.245 \\
          accuracy & 99\% & 98.1\% & 100\% \\
          \bottomrule
        \end{tabular}
      \end{sc}
    \end{small}
  \end{center}
  \vskip -0.1in
\end{table}
A final Max-SCCE loss of approximately $0.426$ corresponds to perfect classification under the max-loss criterion (\ref{prop:maxloss}), while maintaining moderate confidence levels. 
In contrast, the MLP attains low average loss at the cost of highly confident predictions, resulting in substantially larger maximum loss values as shown in Table~\ref{tab:iris_test}.

\begin{table}[h!]
  \caption{\textbf{Testing} performance on the Iris dataset.}
  \label{tab:iris_test}
  \begin{center}
    \begin{small}
      \begin{sc}
        \begin{tabular}{lcccr}
          \toprule
         Model& MLP ($\frac{1}{N}$)&LMM ($\frac{1}{N}$) & LMM ($\max_{n}$)  \\
          \midrule
          max loss & \textbf{8.744} & 1.253 & 1.03 \\
          avg loss & 0.047& 0.227  &  0.311 \\
          accuracy & 88\% & 91.11\%  & 93.33\%\\
          \bottomrule
        \end{tabular}
      \end{sc}
    \end{small}
  \end{center}
  \vskip -0.1in
\end{table}
Overall, the results reveal a clear qualitative difference between the two models. While the MLP achieves competitive average loss values, it does so at the cost of severe overconfidence, reflected in its large maximum loss. In contrast, the LMM maintains controlled confidence levels, yielding lower maximum loss values. This controlled expressivity is a desirable property in practice, particularly in settings where robustness and interpretability are important.

Finally, increasing the number of hidden neurons to match the number of training samples (e.g., \( H = 150 \) for Iris) leads to exact interpolation, with the maximum loss converging to zero. This result highlights the expressive power of the LMM model under sparse training, achieved without relying on dense gradient updates.

\subsection{Effect of Weight Initialization on IRIS}
 Due to the sensitivity of morphological networks to weight initialization~\cite{InfluenceofInitialization}, we analyze how different initialization strategies affect the training behavior and final performance of the LMM model. We consider three initialization strategies for the weight matrices $(W^0, W^1, W^2)$: (i) a structured initialization inspired by the theoretical construction of the LMM model in section \ref{sec:init},
    (ii) Gaussian random initialization drawn from $\mathcal{N}(0,k^2)$,
    (iii) Uniform random initialization drawn from $\mathcal{U}(-k, k)$.
Apart from the initialization scheme, all experimental conditions are kept identical.
To account for variability induced by random initialization, each configuration is repeated over 10 independent runs with different random seeds.
The distribution of final Max-SCCE values across runs is reported in the following figure.

\begin{figure}[htb]
  \begin{center}
    \centerline{\includegraphics[width=\columnwidth]{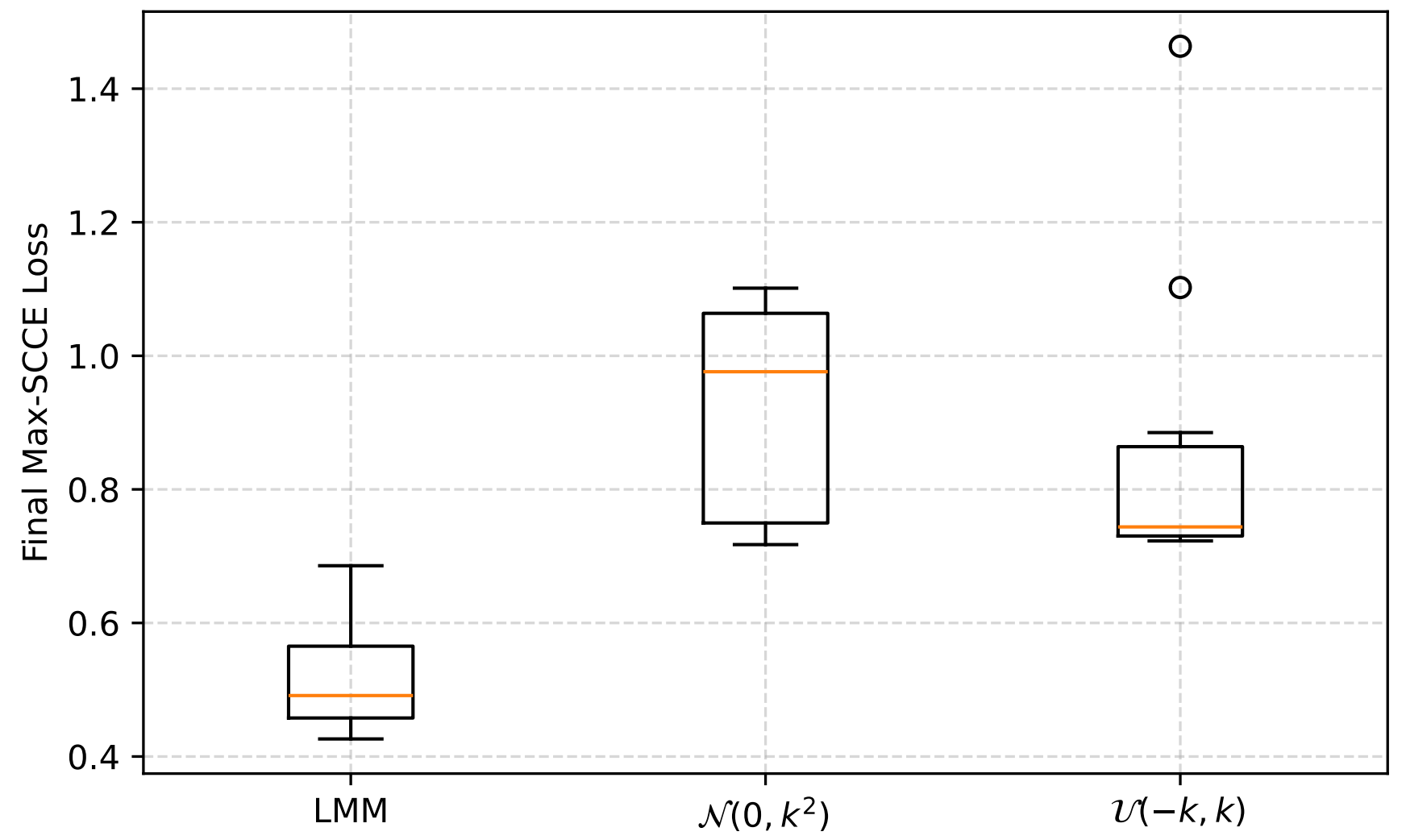}}
    \caption{Final Max-SCCE values on the IRIS dataset obtained with three different weight initialization strategies, evaluated over 10 independent runs.}
    \label{icml-historical}
  \end{center}
\end{figure}

The results show a clear gap between the structured initialization and the random alternatives. The LMM-based initialization consistently converges to lower final Max-SCCE values and exhibits markedly reduced variability across runs. Remarkably, even the worst-performing run under structured initialization attains a lower final loss than the best outcomes obtained with either Gaussian or Uniform random initialization. Although Iris is a small-scale benchmark, these finding clearly demonstrate the decisive role of theory-driven initialization in guiding sparse subgradient optimization toward favorable regions of the parameter space.

\subsection{Training the LMM on MNIST}
To assess scalability, we extend the LMM model with the proposed sparse subgradient algorithm to the MNIST dataset. 
\paragraph{Experimental setup.}
All MNIST experiments were implemented on a CPU cluster using 30 processors. All variants are evaluated under the same configuration, with $n=60{,}000$ training samples, $H=500$ hidden neurons, a skip ratio of 100 iterations for updates $W^0$, and $k=4$. 
Figure~\ref{fig:mnist_convergence} shows the convergence of the LMM model's max-SCCE loss across 200{,}000 iterations.
\begin{figure}[H]
\centering
\includegraphics[width=0.48\textwidth]{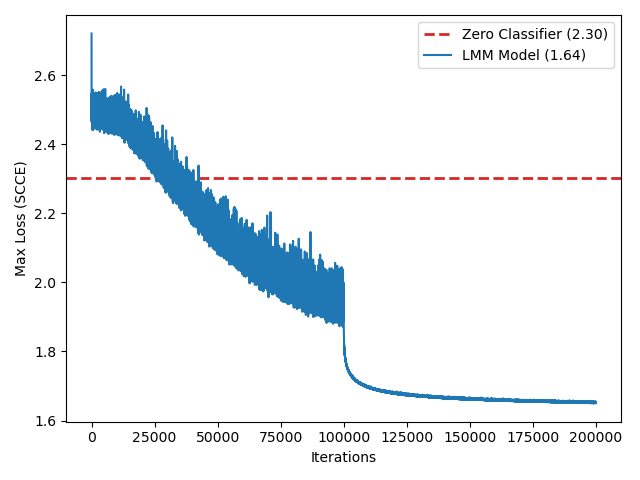}
\caption{Convergence of the LMM model on MNIST}
\label{fig:mnist_convergence}
\end{figure}
We adopt a two-phase step size strategy during training. In the first 100,000 iterations, we employ a Polyak step size, which is well known for its ability to take larger adaptive steps when far from optimality and to avoid shallow local minima in non-smooth optimization landscapes.

In the second phase (starting at iteration 100{,}000), we switch to a constant step size given by $\alpha = \frac{\varepsilon}{2\sqrt{2C}}$,
where $\varepsilon$ is a user-defined learning rate and $C$ is the number of output classes. This fixed learning rate leads to a smoother convergence profile, further reducing the Max-SCCE loss to approximately 1.64. This value significantly outperforms the zero-classifier baseline loss of $\log(C) \approx 2.30$ (red dashed line), which corresponds to uniform predictions across classes.

Quantitive training and testing results, including maximum loss, average loss, and classification accuracy, are reported in Appendix \ref{sec:mnist_accuracy}. These results highlight a clear difference in the optimization behavior: minimizing the maximum loss leads to substantially lower worst-case loss values while simultaneously improving classification accuracy on both the training and test sets. This confirms that maximum-loss minimization constitutes a more effective training objective than average-loss minimization for the LMM on large-scale datasets.

To better understand the nature of the model's predictions, we analyze the confusion matrix on the MNIST test set in Table~\ref{tab:mnist_confusion}. The resulting macro-averaged F1-score reaches $0.89$, indicating balanced classification performance despite the model’s focus on minimizing worst-case errors.

\begin{table}[htb]
\centering
\vspace{0.5\baselineskip}  
\caption{Confusion matrix of the LMM model on the MNIST test set ($88.6\%$ accuracy).}
\vspace{0.5\baselineskip}  
\label{tab:mnist_confusion}
\resizebox{0.48\textwidth}{!}{%
\begin{tabular}{c|cccccccccc}
\toprule
 & 0 & 1 & 2 & 3 & 4 & 5 & 6 & 7 & 8 & 9 \\
\midrule
0 & 922 & 0 & 8 & 5 & 3 & 26 & 5 & 2 & 5 & 4 \\
1 & 0 & 1089 & 12 & 2 & 8 & 2 & 3 & 4 & 14 & 1 \\
2 & 9 & 1 & 923 & 30 & 11 & 44 & 21 & 25 & 4 & 1 \\
3 & 7 & 3 & 31 & 849 & 0 & 64 & 0 & 9 & 30 & 17 \\
4 & 4 & 2 & 3 & 4 & 820 & 1 & 5 & 8 & 14 & 121 \\
5 & 8 & 0 & 2 & 42 & 3 & 784 & 14 & 8 & 20 & 11 \\
6 & 10 & 5 & 3 & 0 & 28 & 31 & 868 & 1 & 9 & 3 \\
7 & 2 & 4 & 21 & 16 & 6 & 2 & 0 & 930 & 4 & 43 \\
8 & 5 & 7 & 13 & 32 & 21 & 23 & 8 & 9 & 822 & 34 \\
9 & 8 & 7 & 6 & 11 & 25 & 15 & 2 & 50 & 28 & 857 \\
\bottomrule
\end{tabular}%
}
\end{table}

Taken together, these results highlight the dual strengths of the LMM architecture: (i) its ability to efficiently optimize non-smooth max-loss objectives while maintaining high accuracy, and (ii) its tendency to produce moderate, well-distributed confidence scores rather than overconfident predictions (\ref{confidence-histogram}), a property that is particularly desirable in applications where robustness and uncertainty awareness are as important as achieving high accuracy.

\subsection{Sparse vs. Dense Updates: Computational Cost per Iteration}

A central motivation for the proposed training procedure is to reduce the computational cost of optimizing LMM models by exploiting the sparsity of the subgradient updates. To quantify the practical impact of this design choice, we compare the average time per training iteration under three update strategies: (i) sparse updates, (ii) sparse updates with skipping updates of the input layer parameters $W^0$, and (iii) dense updates.


To focus on runtime rather than final convergence, each method is run for $300$ iterations, and the average time per iteration (in seconds per iteration) is reported. For the sparse and sparse with periodically skip $W^0$ variants, longer runs were feasible and yielded consistent per-iteration costs. In contrast, extending the dense update strategy beyond a few hundred iterations would incur prohibitive computational cost, and we therefore restrict the dense baseline to the same $300$-iteration budget for a fair comparison of per-iteration complexity.
\begin{table}[htb]
\centering
\caption{Average computational time per iteration on MNIST.}
\label{tab:iteration_time}
\begin{tabular}{lrrr}
\toprule
\textbf{Update mode} & \textbf{Time (s)} & \textbf{s / iter} &\textbf{Mem. (GB)}\!\! \\
\midrule
Sparse + skip $W^0$ & 36  &  0.12 & 513\\
Sparse              & 1045   & 3.48 & 513\\
Dense               & 5688   & 18.96 & 10\\
\bottomrule
\end{tabular}
\end{table}

The results in Table~\ref{tab:iteration_time} demonstrate a substantial computational advantage from exploiting sparsity. Skipping some updates of $W^0$ reduces the average time per iteration from $3.48$ seconds to $0.12$ seconds, corresponding to a speed-up of approximately $29\times$ relative to standard sparse updates. Even without skipping $W^0$, sparse updates already yield a significant improvement over dense updates, reducing the per-iteration cost by a factor of about $5.5\times$.

Importantly, we observe no degradation in predictive performance when periodically skipping updates of $W^0$. This indicates that a large fraction of the computational overhead in the sparse variant arises from maintaining and updating quantities associated with the input layer, and that selectively freezing these parameters can yield substantial efficiency gains without compromising accuracy.

\subsection{Limitations and Challenges}

Although we have proved that the LMM model is trainable and demonstrates desirable properties, several limitations remain. 
First, the training time is significantly longer than that of standard MLPs. For example, on the Iris dataset, training the LMM model requires $151$ seconds, compared to $33$ seconds for MLP under the experimental conditions explained in the previous sections. 
This gap is expected: widely used deep learning frameworks such as PyTorch and TensorFlow have undergone years of optimization, while our implementation is an initial prototype. 
Improving computational efficiency and enabling GPU acceleration therefore constitute important directions for future work.

Second, the memory requirements of LMM models are higher, due to the explicit representation of Short Computational Trees (SCTs). 
We plan to study stochastic alternatives to mitigate this effect. The main challenge is to find estimates of the maximum loss that have a low bias and that can be updated at a low cost.

\subsection{Conclusion}

Our experiments demonstrate that LMM models can be trained effectively using the sparse subgradient algorithm, combining theoretical soundness with strong empirical performance. 
On MNIST, the model achieves $92\%$ classification accuracy while optimizing the worst-case loss. 
This confirms that the training procedure is effective both from an optimization perspective, with a consistent decrease of the objective and from a machine learning perspective with competitive predictive accuracy.
Compared to standard MLPs, LMM networks exhibit less overconfidence, as evidenced by the distribution of prediction confidences. 
This combination of interpretability, robustness, and cautious predictions highlights their potential for safety-critical applications, particularly in medical domains where the cost of overconfident mistakes is unacceptable.

\bibliographystyle{icml2026}
\bibliography{paper}

\begin{thebibliography}{15}
\providecommand{\natexlab}[1]{#1}
\providecommand{\url}[1]{\texttt{#1}}
\expandafter\ifx\csname urlstyle\endcsname\relax
  \providecommand{\doi}[1]{doi: #1}\else
  \providecommand{\doi}{doi: \begingroup \urlstyle{rm}\Url}\fi

\bibitem[Bolte \& Pauwels(2021)Bolte and Pauwels]{conservativesetvaluedfields}
Bolte, J. and Pauwels, E.
\newblock Conservative set valued fields, automatic differentiation, stochastic gradient methods and deep learning.
\newblock \emph{Mathematical Programming}, 188\penalty0 (1):\penalty0 19--51, 2021.

\bibitem[Daniely et~al.(2023)Daniely, Srebro, and Vardi]{computationalcomplexitylearningneural}
Daniely, A., Srebro, N., and Vardi, G.
\newblock Computational complexity of learning neural networks: Smoothness and degeneracy.
\newblock \emph{Advances in Neural Information Processing Systems}, 36:\penalty0 76272--76297, 2023.

\bibitem[Dimitriadis \& Maragos(2021)Dimitriadis and Maragos]{morphological_NN}
Dimitriadis, N. and Maragos, P.
\newblock Advances in morphological neural networks: Training, pruning and enforcing shape constraints.
\newblock \emph{ICASSP 2021 - 2021 IEEE International Conference on Acoustics, Speech and Signal Processing (ICASSP)}, pp.\  3825--3829, 2021.

\bibitem[Dimitrova et~al.(2025)Dimitrova, Blusseau, and Velasco-Forero]{InfluenceofInitialization}
Dimitrova, M., Blusseau, S., and Velasco-Forero, S.
\newblock Learning morphological representations of image transformations: Influence of initialization and layer differentiability.
\newblock 2025.

\bibitem[Glorot \& Bengio(2010)Glorot and Bengio]{glorot2010understanding}
Glorot, X. and Bengio, Y.
\newblock Understanding the difficulty of training deep feedforward neural networks.
\newblock In \emph{Proceedings of the thirteenth international conference on artificial intelligence and statistics}, pp.\  249--256. JMLR Workshop and Conference Proceedings, 2010.

\bibitem[LeCun et~al.(2002)LeCun, Bottou, Bengio, and Haffner]{lecun2002mnist}
LeCun, Y., Bottou, L., Bengio, Y., and Haffner, P.
\newblock Gradient-based learning applied to document recognition.
\newblock \emph{Proceedings of the IEEE}, 86\penalty0 (11):\penalty0 2278--2324, 2002.

\bibitem[LeCun et~al.(2015)LeCun, Bengio, and Hinton]{deepLearning}
LeCun, Y., Bengio, Y., and Hinton, G.
\newblock Deep learning.
\newblock \emph{Nature}, 521\penalty0 (7553):\penalty0 436--444, 2015.

\bibitem[Livni et~al.(2014)Livni, Shalev-Shwartz, and Shamir]{dense_NN_computationalCost}
Livni, R., Shalev-Shwartz, S., and Shamir, O.
\newblock On the computational efficiency of training neural networks.
\newblock In Ghahramani, Z., Welling, M., Cortes, C., Lawrence, N., and Weinberger, K. (eds.), \emph{Advances in Neural Information Processing Systems}, volume~27. Curran Associates, Inc., 2014.

\bibitem[Loizou et~al.(2021)Loizou, Vaswani, Laradji, and Lacoste-Julien]{polyakStepSize}
Loizou, N., Vaswani, S., Laradji, I.~H., and Lacoste-Julien, S.
\newblock Stochastic polyak step-size for sgd: An adaptive learning rate for fast convergence.
\newblock In \emph{International Conference on Artificial Intelligence and Statistics}, pp.\  1306--1314. PMLR, 2021.

\bibitem[Luo \& Fan(2021)Luo and Fan]{minmaxplus}
Luo, Y. and Fan, S.
\newblock Min-max-plus neural networks.
\newblock \emph{preprint arXiv:2102.06358}, 2021.

\bibitem[Mondal et~al.(2019)Mondal, Mukherjee, Santra, and Chanda]{morphological_activation}
Mondal, R., Mukherjee, S.~S., Santra, S., and Chanda, B.
\newblock Morphological network: How far can we go with morphological neurons?
\newblock In \emph{British Machine Vision Conference}, 2019.

\bibitem[Nesterov(2014)]{nesterov_subgradient}
Nesterov, Y.
\newblock Subgradient methods for huge-scale optimization problems.
\newblock \emph{Mathematical Programming}, 146\penalty0 (1-2):\penalty0 275--297, 2014.

\bibitem[Sra et~al.(2012)Sra, Nowozin, and Wright]{denseGradient}
Sra, S., Nowozin, S., and Wright, S.~J. (eds.).
\newblock \emph{Optimization for Machine Learning}.
\newblock MIT Press, Cambridge, MA, 2012.

\bibitem[Tsiamis \& Maragos(2019)Tsiamis and Maragos]{sparsity_maxplus_structure}
Tsiamis, A. and Maragos, P.
\newblock Sparsity in max-plus algebra and systems.
\newblock \emph{Discrete Event Dynamic Systems}, 29\penalty0 (1):\penalty0 163--189, 2019.

\bibitem[Zhang et~al.(2019)Zhang, Blusseau, Velasco-Forero, Bloch, and Angulo]{filter_selection}
Zhang, Y., Blusseau, S., Velasco-Forero, S., Bloch, I., and Angulo, J.
\newblock Max-plus operators applied to filter selection and model pruning in neural networks.
\newblock 2019.

\end{thebibliography}

\newpage
\appendix
\onecolumn
\section{Appendix}

\subsection{Proof of Theorem~\ref{thm:universal_approx} with an explicit construction of the sequence of parameters}

\label{ProofThmUniversalApprox}


\begin{proof}
\begin{enumerate}

    \item 
 \textbf{Since $f$ is $K$–Lipschitz under $\|\cdot\|_\infty$, we have}
\[
\forall x,y \in X, \qquad |f(x) - f(y)| \le K \|x-y\|_\infty.
\]

\item \textbf{Define the linear transformation} $\lambda(x)$ as
\[
\lambda(x) = [Kx_1, \, -Kx_1, \, \ldots, \, Kx_p, \, -Kx_p] \in \mathbb{R}^{2p},
\]
that is,
\[
\lambda : \mathbb{R}^p \to \mathbb{R}^{2p}.
\]

By using both sides ($\pm Kx_j$), we ensure that the next layer of (min,+) can shape a pyramid centered at any point, which is crucial for approximating Lipschitz functions.

\item \textbf{Define $(\min,+)$ functions $g_i(x)$, one for each neuron $i$}, as
\[
g_i(x) = \min_{j=1}^{2p} \big( \lambda_j(x) + a_{ij} \big),
\]
where $a_i \in \mathbb{R}^{2p}$ is a vector of learnable biases (one per neuron).  

Each $g_i$ is a concave piecewise linear function. It has a peak at some point $x^i$, provided we set $a_{ij}$ correctly. In this way, $g_i$ defines a pyramid function centered at the chosen grid point $x^i$.

\item  \textbf{Constructing the pyramid at grid points:} \\
Let $f : \mathbb{R}^p \to \mathbb{R}$ be the target Lipschitz function.  
Define a set of grid points 
\[
\mathcal{M} = \{ x^1, x^2, \ldots, x^m \} \subset \mathbb{R}^p,
\]
where we know the values of $f(x^i)$.  

The grid spacing is denoted by $\delta > 0$, and we let $\delta \to 0$.  

We assume that every $x \in X$ is close to some grid point $x^i$. At each grid point $x^i$, we want to construct a $(\min,+)$ function $g_i(x)$ such that $g_i(x) = \min_{j=1}^{2d} (\lambda_j (x) + a_{ij})$ and satisfying 
\[
g_i(x^i) = f(x^i).
\]
This ensures that the pyramid touches the function $f$ at its center.
\begin{enumerate}
    \item Construction of biases $a_i \in \mathbb{R}^{2p}$ for fixed grid points $x^i \in \mathbb{R}^p$

We define the biases as
\[
a_{ij} := -\lambda_j(x^i) + f(x^i).
\]
With this choice,
\[
g_i(x) = \min_{j=1}^{2p} \big( \lambda_j(x) + a_{ij} \big)
       = \min_{j=1}^{2p} \big( \lambda_j(x) - \lambda_j(x^i) + f(x^i) \big).
\]
In particular, at $x = x^i$ we obtain
\[
g_i(x^i) = f(x^i).
\]
Thus, this definition of $a_{ij}$ ensures that the pyramid function $g_i(x)$ touches $f(x)$ at the grid point $x^i$ with value exactly equal to $f(x^i)$.

\medskip
Since $\lambda(x) = [Kx_1, -Kx_1, \ldots, Kx_d, -Kx_d]$, we have
\[
\lambda_j(x) = Kx_j, \quad j=1,\ldots,p,
\qquad
\lambda_j(x) = -Kx_{j-p}, \quad j=p+1,\ldots,2p.
\]
Hence
\[
a_{ij} = -Kx^i_j + f(x^i), \quad j=1,\ldots,p,
\qquad
a_{ij} = Kx^i_{j-d} + f(x^i), \quad j=p+1,\ldots,2p.
\]

Substituting into $g_i(x)$ gives
\[
g_i(x) = \min_{j=1}^{2p} \big( \lambda_j(x) + a_{ij} \big)
       = \min \Big\{ \min_{j=1,\ldots,p} \big( K(x_j - x^i_j) + f(x^i) \big), \;
                     \min_{j=p+1,\ldots,2p} \big( -K(x_{j-p} - x^i_{j-p}) + f(x^i) \big) \Big\}.
\]
Therefore,
\[
g_i(x) = f(x^i) - K \| x - x^i \|_\infty.
\]

\medskip
This shows that $g_i(x)$ is a tent (or pyramid) function, i.e. an absolute-value cone flipped and shifted upward:
\begin{enumerate}
    \item  it has a peak at $x^i$,  
\item  it reaches the height $f(x^i)$,  
\item  it slopes down on both sides with slope $\pm K$.  
\end{enumerate}
\medskip

At each $x^i \in \mathcal{M}$ we build a pyramid $g_i(x)$ such that $g_i(x^i) = f(x^i)$.  
This construction forces the approximation to match the true value of $f$ at the grid points.

\end{enumerate}

\item  \textbf{Constructing the final approximation} $h(x)$\\

We have built a family of tent (pyramid) functions $g_i(x)$, each one centered at a grid point $x^i \in \mathcal{M}$ and satisfying
\[
g_i(x^i) = f(x^i).
\]

Now we define the approximation function $h(x)$ as the maximum over all these pyramids:
\[
h(x) = \max_{j=1,\ldots,2p} g_j(x).
\]

We take the maximum because each pyramid $g_i(x)$ is localized: it approximates $f$ well only near its center $x^i$. By taking the maximum over all pyramids, we combine their strengths.  
\begin{itemize}
    \item Each point $x$ is close to some grid center $x^i$.  
\item The closest pyramid provides the best local estimate of $f(x)$.  
\item The maximum ensures that $h(x)$ selects the highest (best) local approximation at each point.  
\end{itemize}

\medskip
\noindent\textbf{At a grid point $x^i$:}  
\[
h(x^i) = \max_{j=1,\ldots,2p} g_j(x^i) 
      = g_i(x^i) = f(x^i), \qquad \text{for } j=i, \tag{1}
\]
and 
\[
g_j(x^i) \leq f(x^i), \qquad \forall j \neq i. \tag{2}
\]

\medskip
\noindent\textbf{Verification.}  
Since $f$ is $K$–Lipschitz,
\[
|f(x^j) - f(x^i)| \le K \|x^j - x^i\|_\infty.
\]
In particular,
\[
f(x^j) \le f(x^i) + K \|x^j - x^i\|_\infty. 
\]
By construction,
\[
g_j(x^i) = -K \|x^j - x^i\|_\infty + f(x^j).
\]
Combining with (1) gives
\[
g_j(x^i) \le -K \|x^j - x^i\|_\infty + f(x^i) + K \|x^j - x^i\|_\infty
           = f(x^i).
\]
Hence $g_j(x^i) \le f(x^i)$ for all $j \ne i$.

\medskip
\noindent\textbf{Relation between norms.}  
For all $v \in \mathbb{R}^n$,
\[
\|v\|_\infty \;\le\; \|v\|_1 \;\le\; n \|v\|_\infty.
\]
Thus, if $f$ is Lipschitz continuous under $\|\cdot\|_\infty$, it is also Lipschitz continuous under $\|\cdot\|_1$ (and vice versa), with the same constant up to a factor depending only on $n$.

Thus the upper envelope of all pyramids satisfies
\[
h(x^i) = \max_j g_j(x^i) = g_i(x^i) = f(x^i), \quad \forall x^i \in \mathcal{X}.
\]

This means that the approximation $h(x)$ matches the function $f(x)$ exactly at every grid point.

\item \textbf{Approximation error outside the grid points
}\\
We know that the approximation $h(x)$ matches the function $f(x)$ exactly at the grid points.  
Now we want to show that even outside the grid points, our approximation $h(x)$ remains very close to the true $f(x)$.

\medskip
\noindent\textbf{Step 6.1: Lipschitz control.}  
Since $f$ is $K$–Lipschitz, for any $x,x^i \in X$ we have
\[
|f(x^i) - f(x)| \le K \| x^i - x \|_\infty. 
\]

\medskip
\noindent\textbf{Step 6.2: Error bound near grid points.}  
Fix $x \in X$, and let $x^i$ be the closest grid point to $x$, so that
\[
\|x - x^i\|_\infty \le \delta.
\]

Consider the pyramid function centered at $x^i$:
\[
g(x) = -K \|x - x^i\|_\infty + f(x^i).
\]

Since $h(x) = \max_j g_j(x)$, we have
\[
h(x) \ge g(x).
\]

Therefore,
\[
h(x) - f(x) \;\ge\; g(x) - f(x)
               = f(x^i) - K \|x - x^i\|_\infty - f(x).
\]

Using the Lipschitz inequality (A),
\[
f(x^i) - f(x) \;\ge\; -K \|x^i - x\|_\infty,
\]
hence
\[
h(x) - f(x) \;\ge\; -K \|x - x^i\|_\infty - K \|x - x^i\|_\infty
                 = -2K \|x - x^i\|_\infty
                 \ge -2K\delta. 
\]

\medskip
\noindent\textbf{Step 6.3: Upper bound.}  
We also know that $h(x) \le f(x)$, since each pyramid lies below $f$ by construction.  
Thus,
\[
-2K\delta \;\le\; h(x) - f(x) \;\le\; 0.
\]

Equivalently,
\[
|h(x) - f(x)| \le 2K\delta.
\]

\medskip
\noindent\textbf{Conclusion.}  
The approximation $h(x)$ equals $f(x)$ exactly on the grid points, and differs from $f(x)$ by at most $2K\delta$ everywhere else.  
As the grid spacing $\delta \to 0$, we obtain
\[
\sup_{x \in X} |h(x) - f(x)| \to 0,
\]

which establishes the uniform convergence.
\end{enumerate}
\end{proof}

\subsection{Extension to classification}
\label{supp:pf-vs-exp}
In the detailed proof of Theorem~\ref{thm:universal_approx} the construction is given for scalar-valued functions 
\(f : \mathbb{R}^p \to \mathbb{R}\). 
There, each pyramid function is defined as
\[
g_i(x) = \min_j \big( \lambda_j(x) + a_{ij} \big),
\qquad 
h(x) = \max_i g_i(x),
\]
which provides a scalar approximation of \(f\).  

In our experimental setting, however, we deal with vector-valued functions 
\(f : \mathbb{R}^p \to \mathbb{R}^C\), corresponding to the $C$ output classes in a multi-class problem.  
A direct extension of the proof would require building \(C \times m\) pyramids, one per class and per grid point, followed by an additional output layer. This quickly becomes prohibitive as both $C$ (number of classes) and $m$ (number of samples) increase.

Instead, we adopt a more efficient strategy inspired by the proof but tailored to classification.  
We first construct a single bank of $m$ shared pyramids \(g_j(x)\), exactly as in the scalar case, each centered at a grid point $x^j$.  
At a training point $x^i$ with true label $y_i$, these pyramids satisfy
\[
g_j(x^i) = \mathbf{1}_{ij},
\]
so that each pyramid encodes its center.  

To obtain class-specific scores, we reuse these shared pyramids but apply simple constant shifts controlled by a confidence parameter $k>0$.  
For each class $d \in \{1,\dots,d\}$, we define
\[
h^d(x) = \max_j \Big( g_j(x) + k \,(2 \cdot \mathbf{1}_{\{d=y_j\}} - 1) \Big),
\]
which means that if $d$ is the true class for a sample $x^j$, the score is boosted by $+k$, whereas if $d \neq y_j$, the score is penalized by $-k$.  

At a grid point $x^i$, this construction guarantees that
\[
h^{y_i}(x^i) \ge k, 
\qquad 
h^d(x^i) \le -k \quad \text{for all } d \ne y_i,
\]
so that the predicted class is exactly the true label: 
\(\arg\max_d h^d(x^i) = y_i\).  

This modification is crucial: rather than duplicating $m$ pyramids for each of the $C$ classes (totaling $C \times m$ neurons), we only need $m$ shared pyramids plus $C$ class-specific shifts.  
The resulting complexity is therefore $C+m$, a dramatic reduction that makes the method scalable while preserving the constructive spirit of the proof.  
In particular, the experimental model inherits the geometric intuition of the theoretical construction, pyramids centered at the samples, but achieves it with far fewer operations.

\subsection{Proof of the proposition}
\begin{proof}
\label{proposition_proof}
Let $\{(x_n,y_n)\}_{n=1}^N$ be a labeled dataset with $y_n\in\{1,\dots,C\}$, and suppose the model produces logits $Z_d(x_n)$ for each class $d$.  
The softmax probabilities are
\[
\hat{Y}_{d}(x_n)=\frac{\exp(Z_{d}(x_n))}{\sum_{d'=1}^{C}\exp(Z_{d'}(x_n))},\]
\[
\mathcal{L}(x_n,y_n)= -\log \hat{Y}_{y_n}(x_n)
\]

is the sparse categorical cross-entropy loss.

Assume, for contradiction, that the model does not achieve perfect classification accuracy.
Then, there exists a sample $x_n$ and an incorrect class $d' \ne y_n$ such that $Z_{d'}(x_n) \ge Z_{y_n}(x_n)$ and thus $e^{Z_{d'}(x_n)} \ge e^{Z_{y_n}(x_n)}$.

Therefore, the softmax denominator satisfies:
\[
\sum_{d=1}^C e^{Z_d(x_n)} \ge e^{Z_{y_n}(x_n)} + e^{Z_{d'}(x_n)} \ge 2e^{Z_{y_n}(x_n)}.
\]

Hence, the predicted probability for the true class is:
\[
\hat{Y}_{y_n}(x_n) = \frac{e^{Z_{y_n}(x_n)}}{\sum_{d=1}^C e^{Z_d(x_n)}}
\le \frac{e^{Z_{y_n}(x_n)}}{2e^{Z_{y_n}(x_n)}} = \frac{1}{2}.
\]

Taking the negative logarithm gives:
\[
\mathcal{L}(x_n, y_n) = -\log \hat{Y}_{y_n}(x_n) \ge \log 2.
\]
whch contradicts $\max_n \mathcal{L} < \log 2$.

Therefore, for every training sample we must have
\[
\arg\max_d Z_d(x_n) = y_n,
\]
which implies the model predicts every training label correctly,  
i.e., the training accuracy is $100\%$.
\end{proof}
\subsection{Proof of the Subgradient Formula for 0-hidden layer}
\label{supp:subgrad-proof}
\begin{proof}

The sparse categorical cross entropy loss is given by: \\
\[
\text{Loss}_n(W) = -\log\left(\hat{Y}_{n, y_{n}}\right)
\]
where after replacing $\hat{Y}_{n, y_{n}}$ by its formula \eqref{formula_Yhat_0-hidden}, we get: 
\[
\text{Loss}_n(W)= -\log\left(
  \frac{
    \exp\!\Big(\max_{p}\big(X_{n,p} + W_{d,p}\big)\Big)
  }{
    \sum_{d=0}^{9} \exp\!\Big(\max_{p}\big(X_{n,p} + W_{d,p}\big)\Big)
  }
\right)
\]\\
We wish to calculate a subgradient of
\[
\mathcal L(W) = \max_n  \text{Loss}_n(W) = \max_n \left( -\max_p (X_{n,p} + W_{y_{n},p}) + \log \sum_{d=0}^{9} \exp \left( \max_p (X_{n,p} + W_{d,p}) \right) \right)
\]
Let us denote
\[n^* = \arg\max_n \left( -\max_p (X_{n,p} + W_{y_{n},p}) + \log \sum_{d=0}^{9} \exp \left( \max_p (X_{n,p} + W_{d,p}) \right) \right), \] 
\[
p^*(n,d) = \arg\max_p X_{n,p} + W_{d,p}
\]
and $\mathbb I(e) = \begin{cases} 1 & \text{ if $e$ is true,} \\
0 & \text{ if $e$ is false.}\end{cases}$

\begin{align*}
\frac{\partial \mathcal L}{\partial W_{i,j}}(W) &= \frac{\partial}{\partial W_{i,j}} \left[ \max_n \left( -\max_p (X_{n,p} + X_{y_{n},p}) + \log \sum_{d=0}^{9} \exp \left( \max_p (X_{n,p} + W_{d,p}) \right) \right) \right] \\
&= \frac{\partial}{\partial W_{i,j}} \left[ -\max_p \left( X_{n^*,p} + W_{y_{n^*},p} \right) \right] + \frac{\partial}{\partial W_{i,j}} \left[ \log \sum_{d=0}^{9} \exp \left( \max_p X_{n^*,p} + W_{d,p} \right) \right] \\
&= -\frac{\partial}{\partial W_{i,j}} \left[ X_{n^*,p^*(n^*, y_{n^*})} + W_{y_{n^*},p^*(n^*, y_{n^*})}  \right] + \\
&\quad \frac{\sum_{d'=0}^{9} 
\dfrac{\partial}{\partial x_{i,j}} \left[ \max_p(X_{n^*,p}+W_{d',p})\right]* \exp{(X_{n^*, p^*(n^*,d')} + W_{d',p^*(n^*,d')})}}{\sum_{d=0}^{9} \exp \left( X_{n^*,p^*(n^*,d)} + W_{d,p^*(n^*,d)} \right)} \\
&= -\mathbb{I}(i = y_{n^*}) \mathbb{I}(j = p^*(n^*, y_{n^*}))  + \\
&\quad \frac{\sum_{d'=0}^{9} \dfrac{\partial}{\partial W_{i,j}} \left[ X_{n^*,p(n^*,d')}+W_{d',p(n^*,d')})\right] * \exp{(X_{n^*, p^*(n^*,d')} + W_{d',p^*(n^*,d')})}}{\sum_{d=0}^{9} \exp \left( X_{n^*,p^*(n^*,d)} + W_{d,p^*(n^*,d)} \right)} \\
&= -\mathbb{I}(i = y_{n^*}) \mathbb{I}(j = p^*(n^*, y_{n^*}))  + \\
&\quad \frac{\sum_{d'=0}^{9} \mathbb I(i=d') \mathbb I (j = p^*(n^*,d')) * \exp{(X_{n^*, p^*(n^*,d')} + W_{d',p^*(n^*,d')})}}{\sum_{d=0}^{9} \exp \left( X_{n^*,p^*(n^*,d)} + W_{d,p^*(n^*,d)} \right)} \\
&= -\mathbb{I}(i = y_{n^*}) \mathbb{I}(j = p^*(n^*, y_{n^*})) + \\
&\quad  \frac{\exp \left( X_{n^*,p^*(n^*,i)} + W_{i,p^*(n^*,i)} \right) \mathbb{I}(j=p^*(n^*,i))}{\sum_{d=0}^{9} \exp \left( X_{n^*,p^*(n^*,d)} + W_{d,p^*(n^*,d)} \right)}
\end{align*}

In total, the sparse subgradient matrix will have 10 nonzeros elements: for each $i \in \{0, \ldots, 9\}$,
the element $(i, p^*(n^*, i))$ and the element $(y_{n^*}, p^*(n^*, y_{n^*}))$, which is already in the previous list.
\end{proof}

\subsection{Proof of the Subgradient Formula for the LMM Model}
\label{supp:subgrad-proof-LMM}

\begin{proof}

Let $\text{Loss}_n(W) = -\log (\hat{Y}_{y_n}(x_n))$ be the sparse categorical cross-entropy loss, where the predicted probability is
\[
\hat{Y}_{d}(x) = 
\frac{
\exp(Z_d(x))
}{
\sum_{d'=1}^{C} \exp(Z_{d'}(x))
}
\]
and
\[
Z_d(x) = \max_{h} \left( g_h(x) + W^2_{h,d} \right), \quad
g_h(x) = \min_{i} \left( \lambda_i(x) + W^1_{i,h} \right), \quad
\lambda(x) = W^0 x.
\]

We define the training objective:
\[
\mathcal L(W) = \max_n \text{Loss}_n(W) = \max_n \left( -\log \hat{Y}_{y_n}(x_n) \right).
\]

Let $n^\star \in \arg\max_n \text{Loss}_n(W)$ be a worst-case sample.  
We define:
\[
h_d^\star \in \arg\max_h \left( g_h(x_{n^\star}) + W^2_{h,d} \right),
\quad
i_d^\star \in \arg\min_i \left( \lambda_i(x_{n^\star}) + W^1_{i,h_d^\star} \right),
\]
and use the shorthand $\mathbb{I}(e) = 1$ if $e$ is true, 0 otherwise.

We now expand the loss function fully:
\[
\text{Loss}_{n^\star}(W) 
= -\log \left( \frac{ \exp(Z_{y_{n^\star}}(x_{n^\star})) }{ \sum_{d=1}^C \exp(Z_d(x_{n^\star})) } \right)
= -Z_{y_{n^\star}}(x_{n^\star}) + \log \sum_{d=1}^C \exp(Z_d(x_{n^\star})).
\]

So:
\[
\mathcal{L}(W) = -Z_{y_{n^\star}}(x_{n^\star}) + \log \sum_{d=1}^{C} \exp(Z_d(x_{n^\star})).
\]

\paragraph{(a) Subgradient with respect to \(W^2\):}

We compute:
\begin{align*}
\frac{\partial \mathcal{L}}{\partial W^2_{h,d}} 
&= \frac{\partial}{\partial W^2_{h,d}} \left[ -Z_{y_{n^\star}}(x_{n^\star}) + \log \sum_{d'=1}^{C} \exp(Z_{d'}(x_{n^\star})) \right] \\
&= - \frac{\partial Z_{y_{n^\star}}(x_{n^\star})}{\partial W^2_{h,d}} + \frac{\partial}{\partial W^2_{h,d}} \left[ \log \sum_{d'=1}^C \exp(Z_{d'}(x_{n^\star})) \right] \\
&= - \mathbb{I}(d = y_{n^\star}) \cdot \mathbb{I}(h = h^\star_d)
\;\; +\;\;
\frac{\exp(Z_d(x_{n^\star}))}{\sum_{d'} \exp(Z_{d'}(x_{n^\star}))} \cdot \mathbb{I}(h = h^\star_d) \\
&= \mathbb{I}(h = h^\star_d) \cdot \left( \hat{Y}_d(x_{n^\star}) - \mathbb{I}(d = y_{n^\star}) \right)
\end{align*}

Hence,
\[
\frac{\partial \mathcal{L}}{\partial W^2_{h,d}} =
\mathbb{I}(h = h_d^\star) \cdot \left( \hat{Y}_d(x_{n^\star}) - \mathbb{I}(d = y_{n^\star}) \right)
\]

\paragraph{(b) Subgradient with respect to \(W^1\):}

We now apply the chain rule through
\[
g_h(x) = \min_i \bigl(\lambda_i(x) + W^1_{i,h}\bigr).
\]

We compute:
\[
\frac{\partial \mathcal{L}}{\partial W^1_{i,h}} 
= \sum_{d=1}^{C} 
\frac{\partial \mathcal{L}}{\partial Z_d(x_{n^\star})}
\cdot \frac{\partial Z_d(x_{n^\star})}{\partial g_h(x_{n^\star})}
\cdot \frac{\partial g_h(x_{n^\star})}{\partial W^1_{i,h}}.
\]

First, using the softmax–cross-entropy expression, we have
\[
\frac{\partial \mathcal{L}}{\partial Z_d(x_{n^\star})} 
= \hat{Y}_d(x_{n^\star}) - \mathbb{I}_{\{d = y_{n^\star}\}}.
\]

Second, from the definition 
\(Z_d(x)=\max_h(g_h(x)+W^2_{h,d})\), we obtain
\[
\frac{\partial Z_d(x_{n^\star})}{\partial g_h(x_{n^\star})} 
= \mathbb{I}_{\{h = h_d^\star\}}.
\]

Third, from \(g_h(x)=\min_i(\lambda_i(x)+W^1_{i,h})\), we have
\[
\frac{\partial g_h(x_{n^\star})}{\partial W^1_{i,h}} 
= \mathbb{I}_{\{i = i_d^\star\}}.
\]

Combining these three expressions yields
\begin{align*}
\frac{\partial \mathcal{L}}{\partial W^1_{i,h}}
&= \sum_{d=1}^{C}
\left[
\left(\hat{Y}_d(x_{n^\star})-\mathbb{I}_{\{d=y_{n^\star}\}}\right)
\mathbb{I}_{\{h = h_d^\star\}} \mathbb{I}_{\{i = i_d^\star\}}
\right].
\end{align*}

Equivalently,
\[
\frac{\partial \mathcal{L}}{\partial W^1[i,h]} =
\begin{cases}
\hat{Y}_d(x_{n^\star})-\mathbb{I}_{\{d=y_{n^\star}\}}
\quad \text{for some $d$ with $h=h_d^\star$ and $i=i_d^\star$,}\\[4pt]
0,\quad \text{otherwise.}
\end{cases}
\]

In particular, because 
\(\frac{\partial \mathcal{L}}{\partial W^2[h,d]}=
\mathbb{I}_{\{h=h_d^\star\}}
(\hat{Y}_d(x_{n^\star})-\mathbb{I}_{\{d=y_{n^\star}\}})\),
we can also write
\[
\frac{\partial \mathcal{L}}{\partial W^1[i,h]} =
\sum_{d=1}^{C}\mathbb{I}_{\{h=h_d^\star\}}\mathbb{I}_{\{i=i_d^\star\}}
\frac{\partial \mathcal{L}}{\partial W^2[h,d]}
\]

\paragraph{(c) Subgradient with respect to \(W^0\):}

We apply the full chain rule:
\[
\frac{\partial \mathcal{L}}{\partial W^0[i,p]} =
\sum_{d=1}^{C}
\frac{\partial \mathcal{L}}{\partial Z_d(x_{n^\star})}
\cdot \frac{\partial Z_d(x_{n^\star})}{\partial g_{h_d^\star}(x_{n^\star})}
\cdot \frac{\partial g_{h_d^\star}(x_{n^\star})}{\partial \lambda_i(x_{n^\star})}
\cdot \frac{\partial \lambda_i(x_{n^\star})}{\partial W^0[i,p]}
\]

We know:
\[
\frac{\partial \mathcal{L}}{\partial Z_d(x)} 
= \frac{\partial \mathcal{L}}{\partial W^2[h_d^\star,d]}
\qquad\qquad
\frac{\partial Z_d}{\partial g_{h_d^\star}} = 1
\]

\[
\frac{\partial g_{h_d^\star}}{\partial \lambda_i} = \mathbb{I}(i = i_d^\star)
\qquad\qquad
\frac{\partial \lambda_i(x)}{\partial W^0[i,p]} = x_{n^\star,p}
\]

So:
\[
\frac{\partial \mathcal{L}}{\partial W^0[i,p]} =
\begin{cases}
x_{n^\star,p} \cdot \frac{\partial \mathcal{L}}{\partial W^2[h_d^\star,d]}, & \text{if } i = i_d^\star\\
0, & \text{otherwise}
\end{cases}
\]

\paragraph{Summary.}
The above derivation shows that the subgradients propagate sparsely along the active paths:
\begin{itemize}
\item 
\(\displaystyle
\frac{\partial \mathcal L}{\partial W^2[h,d]}=
\mathbb{I}_{\{h=h_d^\star\}}\left(\hat{Y}_{d}(x_{n^\star})-\mathbb{I}_{\{d=y_{n^\star}\}}\right)
\).
\item 
\(\displaystyle
\frac{\partial \mathcal L}{\partial W^1[i,h]}=
\frac{\partial \mathcal L}{\partial W^2[h,d]}\quad\text{if }h=h_d^\star\text{ and }i=i_d^\star.
\)
\item 
\(\displaystyle
\frac{\partial \mathcal L}{\partial W^0[i,p]}=
 x_{n^\star,p}\,\frac{\partial \mathcal L}{\partial W^2[h_d^\star,d]}\quad\text{if }i=i_d^\star,
\)
\end{itemize}
\end{proof}

\subsection{confidence-histogram}
\label{confidence-histogram}
\begin{figure}[htb]
  \vskip 0.2in
  \begin{center}
    \centerline{\includegraphics[width=0.5\columnwidth]{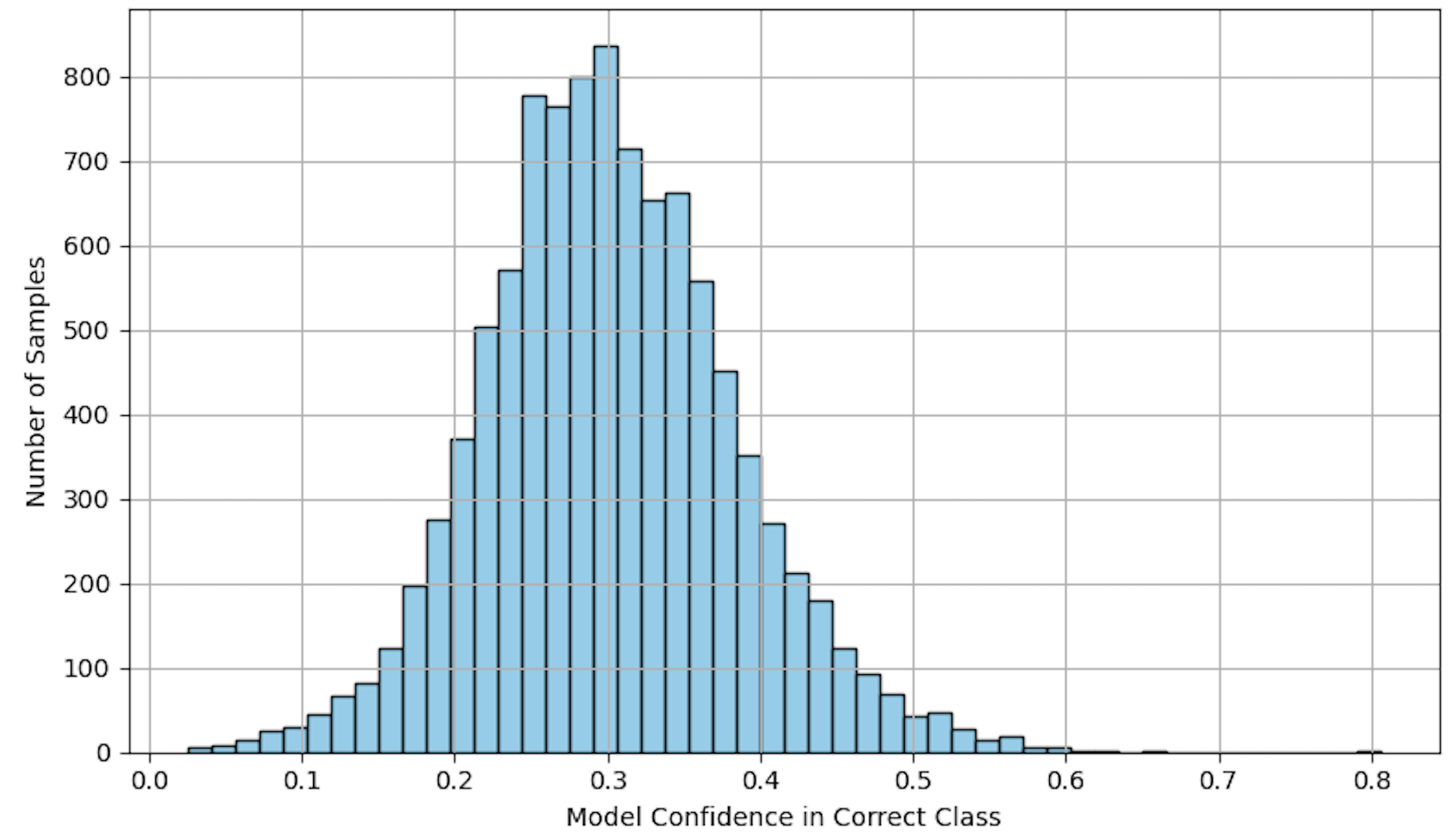}}
    \caption{Distribution of predicted probabilities for the true labels on MNIST using the LMM model.}
    \label{icml-historical}
  \end{center}
\end{figure}

\newpage
\subsection{MLP vs LMM on MNIST}
\label{sec:mnist_accuracy}

\begin{table}[h]
  \caption{\textbf{Training} performance on the MNIST dataset.}
  \label{tab:iris_train}
  \begin{center}
    \begin{small}
      \begin{sc}
        \begin{tabular}{lcccr}
          \toprule
         Model& MLP ($\frac{1}{N}$)&LMM ($\frac{1}{N}$) & LMM ($\max_{n}$)  \\
          \midrule
          max loss & 27.63& 5.02 & 1.64  \\
          avg loss & 0.05 & 1.23 &  1.21 \\
          accuracy & 99.15\% & 69\% & 91.54\% \\
          \bottomrule
        \end{tabular}
      \end{sc}
    \end{small}
  \end{center}
  \vskip -0.1in
\end{table}
\begin{table}[h!]
  \caption{\textbf{Testing} performance on the MNIST dataset.}
  \label{tab:iris_test}
  \begin{center}
    \begin{small}
      \begin{sc}
        \begin{tabular}{lcccr}
          \toprule
         Model& MLP ($\frac{1}{N}$)&LMM ($\frac{1}{N}$) & LMM ($\max_{n}$)  \\
          \midrule
          max loss & \textbf{27.63} & 4.55 & 3.59 \\
          avg loss & 0.047& 1.22  &  1.23 \\
          accuracy & 88\% & 69\%  & 88.75\%\\
          \bottomrule
        \end{tabular}
      \end{sc}
    \end{small}
  \end{center}
  \vskip -0.1in
\end{table}

\end{document}